%% file: main.tex
\documentclass[10pt]{article}

\usepackage{tikz}
\usepackage{tikz-cd}
\usepackage{witharrows}
\usepackage{graphicx,wrapfig,lipsum}   
\usepackage{xspace}

\usepackage{amsmath,amssymb}
\usepackage{amsthm}
\usepackage{mathtools}
\usepackage[noend]{algorithmic}
\usepackage[ruled,vlined]{algorithm2e}
\usepackage{url}
\usepackage{fullpage}
\usepackage{makeidx}
\usepackage{enumerate}
\usepackage[top=1in, bottom=1.25in, left=1in, right=1in]{geometry}
\usepackage{graphicx,float,psfrag,epsfig,caption,subcaption}

\usepackage[pdftex,
	bookmarks=false,
	colorlinks=true,
	urlcolor=red,
	linkcolor=blue,
	citecolor=red,
	pdfstartpage=1,
	pdfstartview={FitH},  
	bookmarksopen=false
	]{hyperref}
\usepackage{epstopdf}
\usepackage{color}
\usepackage{xr}

\usepackage{caption}
\usepackage[utf8]{inputenc}
\usepackage{thm-restate}

\usepackage{scalerel,stackengine}

\usepackage{enumitem}

\DeclareSymbolFont{rsfs}{U}{rsfs}{m}{n}
\DeclareSymbolFontAlphabet{\mathscrsfs}{rsfs}

\numberwithin{equation}{section}

\newtheorem{theorem}{Theorem}

\usepackage{framed}

\input{commands.tex}

\title{On Size-Independent Sample Complexity of ReLU Networks}

\author{Mark Sellke\thanks{Harvard Statistics. \texttt{msellke@fas.harvard.edu}}}

\date{}

\begin{document}

\maketitle

\begin{abstract}
    We study the sample complexity of learning ReLU neural networks from the point of view of generalization. Given norm constraints on the weight matrices, a common approach is to estimate the Rademacher complexity of the associated function class.
    Previously \cite{golowich2020size} obtained a bound independent of the network size (scaling with a product of Frobenius norms) except for a factor of the square-root depth.
    We give a refinement which often has no explicit depth-dependence at all.
\end{abstract}

\section{Introduction}

Given the stunning empirical successes of deep neural networks, a pressing need has emerged to explain their ability to generalize to unseen data.
The traditional approach proceeds via bounds on VC dimension or Rademacher complexity and gives uniform convergence guarantees for a given function class \cite{bartlett2002rademacher,shalev2014understanding}.

The most classical VC bounds for neural networks, see \eg the book \cite{anthony1999neural}, scale with the number of neurons. 
However as early as \cite{bartlett1996valid} it was realized that \emph{scale-sensitive} bounds depending on the size of the weights may be more useful.
Such bounds have the potential to remain valid even for overparametrized networks with more free parameters than the number of training examples.
Further, as already noted in the abstract of \cite{bartlett1996valid}, they are well-motivated by the prevalence of regularization procedures such as weight decay and early stopping.

Recent work on scale-sensitive generalization has focused on Rademacher complexity bounds that scale with the product of operator norms of the weight matrices. Important such results were shown in \cite{bartlett2017spectrally,neyshabur2017pac} using covering number and PAC-Bayes techniques respectively. While one might hope for bounds depending \emph{only} on weight matrix operator norms, the aforementioned estimates depend polynomially on the network depth, and \cite[Theorem 5.1]{golowich2020size} later showed such depth dependence is unavoidable in general.

Surprisingly, \cite{golowich2020size} showed this issue can be largely avoided if one is willing to consider a product of Frobenius norms rather than operator norms; they obtain mild depth dependence of only a square-root factor, which can be removed entirely at the cost of worse decay in the number of samples.
Their approach stems from the natural idea of iteratively peeling off layers and using the Ledoux--Talagrand contraction lemma \cite{ledoux1991probability} to handle each application of the non-linearity $\sigma$. 
A previous work \cite{neyshabur2015norm} used this idea directly and paid exponentially in the depth for repeated use of the contraction lemma; the technical innovation of \cite{golowich2020size} was to apply the contraction lemma inside an auxilliary exponential moment.

We give a refinement of the main result of \cite{golowich2020size} which depends on upper bounds $M_F(i)$ and $M_{\op}(i)$ on both the Frobenius and operator norm of each weight matrix $W_i$. Our bound is never worse, and is fully depth-independent unless $M_{\op}(i)/M_F(i)\approx 1$ for nearly all of the initial layers (i.e. the weight matrices are approximately rank $1$). 
The idea is to use their argument repeatedly along an well-chosen subsequence of the layers and take advantage of improved concentration estimates at the intermediate stages.

Finally we mention that although this work, along with the papers referenced above, apply for essentially arbitrary neural networks, more refined results have been obtained under further assumptions as well as for structured classes of neural networks \cite{arora2018stronger,wei2019data,chen2019generalization,long2019generalization,garg2020generalization}.

\subsection{Problem Formulation and Main Result}
\label{subsec:setup}

A feedforward neural network is a function of the form
\begin{equation}
    \label{eq:NN-def}
    x\mapsto W_D \sigma(W_{D-1}\sigma(\dots W_1 x)).
\end{equation}
Here each $W_i$ is a $w_i\times w_{i-1}$ real weight matrix and $x\in \cX\subseteq \bbR^{w_0}$. The ReLU non-linearity $\sigma(x)=\max(x,0)$ is applied coordinate-wise, and we require $w_D=1$ so the output is a scalar.

Fix Frobenius norm bounds $M_F(1),\dots,M_F(D)$ and operator norm bounds $M_{\op}(1),\dots,M_{\op}(D)$, where without loss of generality $M_{\op}(d)\leq M_F(d)$. 
For each $1\leq d\leq D$, consider the class $\cF_d$ of $d$-layer ReLU neural networks of the form \eqref{eq:NN-def}, such that for all $1\leq m\leq d$ the $m$-th weight matrix $W_m$ satisfies
\begin{align*}
    \|W_m\|_F&\leq M_F(m),
    \\
    \|W_m\|_{\op}&\leq M_{\op}(m).
\end{align*}
We assume $\cX$ is contained in the radius $B$ ball in $\bbR^{w_0}$ so that $\|x\|_2\leq B$ for all $x\in\cX$. The other widths $w_1,\dots,w_{D_1}$ are arbitrary and may vary across $\cF_d$. It will also be convenient to set
\begin{equation}
\label{eq:Rd}
\begin{aligned}
    P_{\op}(d)
    &=
    \prod_{m=1}^{d}
    M_{\op}(m),
    \\
    P_{F}(d)
    &=
    \prod_{m=1}^{d}
    M_{F}(m),
    \\
    R(d)&=P_{\op}(d)/P_F(d)
    .
\end{aligned}
\end{equation}
Note that $1=R(0)\geq R(1)\geq\dots\geq R(D)$.
Next, let $\cR_n(\cF)$ denote the Rademacher complexity of a class $\cF$ of functions $f:\cX\to \bbR$:
\begin{equation}
\label{eq:cG}
\begin{aligned}
    \cR_n(\cF;x_1,\dots,x_n;\vec\eps)
    &=
    \frac{1}{n}
    \sup_{f\in\cF}
    \sum_{i=1}^n
    \eps_i f(x_i);
    \\
    \cR_n(\cF;x_1,\dots,x_n)
    &=
    \frac{1}{2^n}
    \sum_{\vec\eps\in \{\pm 1\}^n}
    \cR_n(\cF;x_1,\dots,x_n;\vec\eps)
    ;
    \\
    \cR_n(\cF)
    &=
    \sup_{x_1,\dots,x_n\in\cX} 
    \cR_n(\cF;x_1,\dots,x_n).
\end{aligned}
\end{equation}

It is well-known (see e.g. \cite[Chapter 26]{shalev2014understanding}) that an upper bound on $\cR_n(\cF)$ implies a uniform generalization guarantee for $\cF$.
Our main result is as follows.

\begin{theorem}
\label{thm:main}
    In the setting of Subsection~\ref{subsec:setup}, we have the Rademacher complexity bound
    \begin{equation}
    \label{eq:main-bound}
    \cR_n(\cF_D)
    \leq
    15Bn^{-1/2}
    P_F(D)
    \sqrt{
    \sum_{d=0}^{D-1}
    R(d)
    }
    .
    \end{equation}
\end{theorem}

Since $\sum_{d=0}^{D-1} R(d)\leq D$, this recovers \cite[Theorem 3.1]{golowich2020size}.
Moreover one expects that generically $R(d)$ decays exponentially with $d$, so their sum can be viewed as ``usually'' constant. Additionally the widths $w_1,\dots,w_{D-1}$ do not enter at all, so we could also directly allow arbitrary width networks in defining $\cF_d$.

\section{Main Argument}

The general version of our improved bound will depend on an arbitrary subsequence $0\leq d_0<d_1<\dots<d_k=D$ of the layers. We will then optimize this sequence based on the values $R(1),R(2),\dots,R(D)$.
We remark that \cite[Proof of Theorem 3.1]{golowich2020size} is essentially the one-step case $(d_0,d_1)=(0,D)$.

\begin{theorem}
\label{thm:composite}
    For any $0=d_0<d_1<\dots<d_k=D$ we have (recall \eqref{eq:Rd}):
\begin{equation}
    \label{eq:composite-bound}
    \cR_n(\cF_D)
    \leq
    5 B n^{-1/2}
    P_F(D)
    \sum_{i=1}^k
    R(d_{i-1})\sqrt{d_i-d_{i-1}}
    .
\end{equation}
\end{theorem}

\begin{proof}
    We assume $B=1$ for simplicity\footnote{Since all functions in $\cF_D$ are positive homogenous of degree $1$, for general $B$ one can just replace $\cX$ by $\cX/B$ in the proof.} and prove inductively in $1\leq j\leq k$ the bound:
    \begin{equation}
    \label{eq:inductive-bound}
    \cR_n(\cF_{d_j})
    \leq 
    5n^{-1/2}
    P_F(d_j)
     \sum_{i=1}^j
    R(d_{i-1})\sqrt{d_i-d_{i-1}}
    \,.
    \end{equation}
    To induct from $d_{j}$ to $d_{j+1}$, we will apply the technique of \cite{golowich2020size} between $d_{j},d_{j+1}$.
    We fix inputs $x_1,\dots,x_n\in\cX$ throughout the proof and define
    \[
    X_{j}=
    \cR_n(\cF_{d_{j}},x_1,\dots,x_n;\vec\eps)
    \]
    which is random since $\vec\eps$ is. Our goal will be to iteratively bound $\bbE[X_j]$.
    Note that if $\|x\|\leq 1$ and $f\in\cF_{d_j}$ then $|f(x)|\leq P_{\op}(d_j)$. By the bounded differences inequality (e.g. \cite[Theorem 6.2]{boucheron2013concentration}), it follows that $X_j-\bbE[X_j]$ is sub-Gaussian with variance proxy $P_{\op}(d_j)^2/n$. In particular for $\lambda_j>0$ to be chosen later, 
    \begin{equation}
    \label{eq:log-exp-bound}
    \log
    \bbE
    \lt[
    \exp
    \lt(
    \frac{\lambda_j (X_j-\bbE[X_j])}{P_F(d_{j})}
    \rt)
    \rt]
    \leq
    \frac{5P_{\op}(d_j)^2 \lambda_j^2}{P_F(d_{j})^2 n}
    \,
    .
    \end{equation}
    Next we apply \cite[Lemma 3.1]{golowich2020size} to peel the layers between $d_j$ and $d_{j+1}$, obtaining:
    \begin{align*}
    \bbE\lt[e^{\lambda_j X_{j+1}/P_F(d_{j+1})}\rt]
    &\leq
    2^{d_{j+1}-d_j}
    \bbE\lt[e^{\lambda_j X_{j}/P_F(d_j)}\rt]
    \\
    \implies 
    \log
    \bbE\lt[e^{\lambda_j X_{j+1}/P_F(d_{j+1})}\rt]
    &\leq 
    \log 
    \bbE\lt[e^{\lambda_j X_{j}/P_F(d_j)}\rt]
    +
    d_{j+1}-d_j
    .
    \end{align*}
    Combining, we find that
    \begin{align*}
    \bbE[X_{j+1}]
    &\leq
    \frac{P_F(d_{j+1})}{\lambda_j}
    \cdot
    \log
    \bbE\lt[e^{\lambda_j X_{j+1}/P_F(d_{j+1})}\rt]
    \\
    &\leq
    \frac{P_F(d_{j+1})}{\lambda_j}
    \lt(
    \log 
    \bbE\lt[e^{\lambda_j X_{j}/P_F(d_j)}\rt]
    +
    d_{j+1}-d_j
    \rt)
    \\
    &\stackrel{\eqref{eq:log-exp-bound}}{\leq}
    \frac{P_F(d_{j+1})}{\lambda_j}
    \lt(
    \frac{\lambda_j \bbE[X_{j}]}{P_F(d_j)}
    +
    \frac{5P_{\op}(d_j)^2 \lambda_j^2}{P_F(d_{j})^2 n}
    +
    d_{j+1}-d_j
    \rt)
    \\
    &\leq
    \frac{P_F(d_{j+1})}{P_F(d_j)}
    \cdot 
    \bbE[X_j]
    +
    \lt(
    \frac{5 P_{\op}(d_j)^2 P_F(d_{j+1}) }{P_F(d_{j})^2  n}
    \cdot
    \lambda_j
    +
    \frac{P_F(d_{j+1})(d_{j+1}-d_j)}{\lambda_j}
    \rt)
    .
    \end{align*}
    Taking 
    \[
    \lambda_j=\frac{P_F(d_{j})n^{1/2}\sqrt{d_{j+1}-d_j}}{2P_{\op}(d_j)}
    \]
    and defining $Y_j=X_j/P_F(d_j)$, we obtain
    \begin{align*}
    \bbE[Y_{j+1}]
    &\leq
    \bbE[Y_j]
    +
    \frac{5P_{\op}(d_j)\sqrt{d_{j+1}-d_j}}{P_F(d_{j})n^{1/2}}
    \\
    &=
    \bbE[Y_j]
    +
    5n^{-1/2} R(d_j) \sqrt{d_{j+1}-d_j} 
    .
    \end{align*}
    This completes the inductive step for \eqref{eq:inductive-bound} and hence the proof.
\end{proof}

\subsection{Optimizing the Choice of Subsequence}

\begin{proof}[Proof of Theorem~\ref{thm:main}]
The result follows from Theorem~\ref{thm:composite} and the inequality $10\sqrt{2}\leq 15$. Indeed, we may take $d_i$ minimal such that $R(d_i)\leq 2^{-i}$, with $d_0=0$ and the last value $d_k$ equal to $D$. 
Then using Cauchy--Schwarz in the second step, and $d_i-d_{i-1}=|\{0\leq d\leq D-1~:~ 2^{-i}< R(d)\leq 2^{-(i-1)} \}|$ in the last:
\begin{align*}
    \sum_{i=1}^k
    R(d_{i-1})\sqrt{d_i-d_{i-1}}
    &\leq
    2\sum_{i=1}^k
    \frac{\sqrt{d_i-d_{i-1}}}{2^i}
    \leq 
    2
    \sqrt{
    \sum_{i=1}^k
    \frac{d_i-d_{i-1}}{2^i}
    \cdot
    \sum_{i=1}^k
    \frac{1}{2^i}
    }
    \\
    &\leq 
    2\sqrt{
    \sum_{i=1}^k
    \frac{d_i-d_{i-1}}{2^i}
    }
    \leq 
    2\sqrt{2\sum_{d=0}^{D-1} R(d)}
    \,.
    \qedhere
\end{align*}
\end{proof}

\small
\bibliographystyle{alpha}
\bibliography{bib}

\end{document}

%% file: commands.tex
\newcommand{\bea}{\begin{eqnarray}}
\newcommand{\eea}{\end{eqnarray}}
\newcommand{\<}{\langle}
\renewcommand{\>}{\rangle}

\newcommand\eg{{\text{e.g.~}}}

\newcommand{\op}{\text{op}}

\def\eps{{\varepsilon}}

\def\cF{{\mathcal F}}

\def\cR{{\mathcal R}}

\def\cX{{\mathcal X}}

\def\<{\langle}
\def\>{\rangle}

\def\b0{{\boldsymbol{0}}}

\renewcommand{\b}{\mathbf{b}}

\def\lt{\left}
\def\rt{\right}

\def\eps{\varepsilon}

\def\bbE{{\mathbb{E}}}

\def\bbR{{\mathbb{R}}}

\def\cF{{\mathcal{F}}}
